\documentclass[review, numbers]{elsarticle} 
\usepackage{fullpage}
\usepackage{amsfonts,amsmath,amssymb,amsthm,bm}
\usepackage{booktabs,threeparttable}
\usepackage{graphicx,psfrag,epsf}
\usepackage{enumerate}
\usepackage{url}
\usepackage{algorithm}
\usepackage{algpseudocode}
\usepackage{helvet}
\usepackage[colorlinks=true,pagebackref,linkcolor=magenta]{hyperref}
\usepackage{color} 
\usepackage{lineno}
\usepackage{chngcntr}
\usepackage{thmtools, thm-restate}

\theoremstyle{plain}

\usepackage{adjustbox}
\usepackage{dsfont}

\begin{document}

\begin{frontmatter}
  \title{Explaining Categorical Feature Interactions Using Graph Covariance and LLMs}
  \author[1]{Cencheng Shen\corref{cor1}}
\cortext[cor1]{Corresponding author}
\ead{shenc@udel.edu}
\affiliation[1]{organization={University of Delaware},
postcode={19716},
city={Newark},
country={United States}}

\author[2]{Darren Edge}
\ead{darren.edge@microsoft.com}
\affiliation[2]{organization={Microsoft Research},
postcode={CB1 2FB},
city={Cambridge},
country={United Kingdom}}

\author[3]{Jonathan Larson}
\ead{jolarso@microsoft.com}
\affiliation[3]{organization={Microsoft Research},
postcode={98052},
city={Redmond},
country={United States}}

\author[4]{Carey E. Priebe}
\ead{cep@jhu.edu}
\affiliation[4]{organization={Johns Hopkins University},
postcode={21218},
city={Baltimore},
country={United States}}


\begin{abstract} 
Modern datasets often consist of numerous samples with abundant features and associated timestamps. Analyzing such datasets to uncover underlying events typically requires complex statistical methods and substantial domain expertise. A notable example, and the primary data focus of this paper, is the global synthetic dataset from the Counter Trafficking Data Collaborative (CTDC) — a global hub of human trafficking data containing over 200,000 anonymized records spanning from 2002 to 2022, with numerous categorical features for each record. In this paper, we propose a fast and scalable method for analyzing and extracting significant categorical feature interactions, and querying large language models (LLMs) to generate data-driven insights that explain these interactions. Our approach begins with a binarization step for categorical features using one-hot encoding, followed by the computation of graph covariance at each time. This graph covariance quantifies temporal changes in dependence structures within categorical data and is established as a consistent dependence measure under the Bernoulli distribution. We use this measure to identify significant feature pairs, such as those with the most frequent trends over time or those exhibiting sudden spikes in dependence at specific moments. These extracted feature pairs, along with their timestamps, are subsequently passed to an LLM tasked with generating potential explanations of the underlying events driving these dependence changes. The effectiveness of our method is demonstrated through extensive simulations, and its application to the CTDC dataset reveals meaningful feature pairs and potential data stories underlying the observed feature interactions.
\end{abstract}


\end{frontmatter}



\section{Introduction}

Modern datasets often contain numerous samples with abundant features (also called attributes or dimensions), each associated with a timestamp. For example, the publicly available Counter-Trafficking Data Collaborative (CTDC) dataset\footnote{\url{https://www.ctdatacollaborative.org/page/global-synthetic-dataset}. All materials of CTDC are copyright of the International Organization for Migration.} - a global hub of human trafficking data compiled from counter-trafficking organizations worldwide - contains over $200,000$ anonymized records spanning 2002 to 2022. Each record represents an individual trafficking case and includes $26$ categorical features, such as gender, age group, citizenship, country of exploitation, and means of exploitation. Most records also include a timestamp indicating the year of registration. While the importance of analyzing human trafficking cases and the CTDC dataset has garnered significant attention \cite{Edge2020, walby2024improving, eryarsoy2024disentangling, AlTammemi2023, Olisah2024}, data exploration in this domain often depends on complex statistical methods, requires substantial manual inspection and visualization, typically addresses only a limited set of features or major events, and is often time-consuming.

This type of dataset can be viewed as temporal categorical data, which are common in practice and a subject of academic interest \cite{Stoffer1993,Kunihama2013,Li2022}. Traditional time-series analysis assumes paired features $(X_1, Y_1)$, $(X_2, Y_2)$, $\ldots$, $(X_T, Y_T)$, where $X_t$ and $Y_t$ are univariate time series indexed by $t$. However, CTDC and many real-world datasets deviate from this framework. Instead, they involve multivariate temporal data of the form $(\mathbf{X}_{iu},\tau_i)$, where $i = 1,\ldots,n$ indexes samples, $u = 1,\ldots,p$ represents the $p$-dimensional categorical features, and $\tau_i \in \{1,\ldots,T\}$ denotes the timestamp of the $i$-th observation. Unlike classical time-series data, multiple observations may share the same timestamp, introducing irregularity in the sampling structure. This broader setting is common in modern data collection, where each sample is tagged with a timestamp rather than observed at regular intervals.

To identify feature interactions in temporal data, traditional time-series techniques, such as cross-covariance and cross-correlation, are typically used to examine linear relationships \cite{enders2010applied,shumway2010time,box2015time}. However, these methods are limited to evenly-spaced time-series data and are designed only to detect linear relationships. Detecting more general forms of dependence requires universal consistent dependence measures. Two landmark approaches in this domain are distance correlation \cite{SzekelyRizzoBakirov2007} and the Hilbert-Schmidt independence criterion \cite{GrettonEtAl2005}, followed by many others \cite{HellerGorfine2013,MGCDCor,Chatterjee2021} and time-series adaptations \cite{Zhou2012, hsicts, DCorTemporalTest}. One key challenge lies in adapting dependence measures to accommodate irregular temporal samples, rather than being restricted to evenly-sampled time-series data. Another important hurdle is the computational burden, as universal consistent dependence measures typically involve $O(n^2)$ time complexity. Recent advancements in simplifying dependence measures for standard categorical data \cite{Edelmann2022, DCorKSampleTest, zhang2024} suggest the potential for more efficient methods specifically for temporal categorical data.

In this paper, we propose a fast and scalable statistical analysis tool for temporal categorical data, combining a graph covariance approach to efficiently detect feature dependence with the power of large language models (LLMs) to generate immediate data insights. Our method transforms categorical features into binary variables via one-hot encoding and computes a graph covariance matrix for feature pairs at each timestamp. We derive a significance level to extract feature pairs with notable interactions over time or sudden spikes at specific timestamps. Extracted features, along with their timestamps, are then passed to an LLM to uncover potential events or contexts driving these
interactions.

Theoretically, we prove that graph covariance is a consistent dependence measure between Bernoulli random variables. It converges asymptotically to zero if and only if two binary features are independent, enabling non-negative values to quantify the strength of their relationship. Computationally, the method is highly efficient, processing the CTDC global synthetic dataset in under $2$ seconds, compared to hours required by other dependence measures. This efficiency facilitates dynamic analysis of large datasets. The integration of LLMs \cite{vaswani2017attention,Devlin2019BERTPO,Raffel2019ExploringTL,radford2019languagegpt2} further enhances the workflow, providing context-rich insights into temporal feature interactions. While LLM explanations may not always be fully accurate and require validation from domain experts, they offer a valuable starting point for practitioners to interpret data narratives with minimal effort.

We validate our method through simulation studies, demonstrating its capability to detect temporal dependence changes under various scenarios. Its application to the CTDC global synthetic dataset showcases the method's efficiency in analyzing large datasets, identifying feature pairs with significant interactions either across all time points or at specific timestamps, and uncovering the underlying events driving these dependence changes. All experiments were conducted on a local desktop running MATLAB 2024a and Python, on Windows 10, with an Intel 16-core CPU and 64GB of memory. A preliminary experimental tool utilizing OpenAI’s API is available on GitHub\footnote{\url{https://github.com/microsoft/intelligence-toolkit}}. The appendix includes theorem proofs, additional simulation figures, and further details on the LLM queries conducted for the CTDC data.

\section{Main Method}

Given a dataset represented as $(\mathbf{X}_{iu}, \tau_i)$, where $i = 1, \ldots, n$ indexes the samples, $u = 1, \ldots, p$ indexes the features, and $\tau_i \in \{1, \ldots, T\}$ denotes the timestamp of the $i$-th sample. Suppose the $u$-th feature has values in the range $[0, K_u]$, where $K_u$ is a positive integer, and the feature is binary if and only if $K_u = 1$. Let $m$ represent the desired number of feature pairs to be extracted and analyzed. The method proceeds as follows:

\begin{itemize}
\item \textbf{Step 1 (Transform Categorical to Binary)}: For each categorical feature $u$ that is not binary, we transform the feature into binary format using one-hot encoding. Specifically, compute $\mathbf{Z}_u \in [0, 1]^{n \times (K_u + 1)}$ as:
\begin{align*}
\mathbf{Z}_{u}(i, k+1)=1 \mbox{ if and only if } \mathbf{X}_{iu}= k
\end{align*}
for each $i$. This transformation ensures that each row of $\mathbf{Z}_u$ contains all zeros except for a single element. After this step, the $u$-th column is removed from $\mathbf{X}$, and $\mathbf{Z}_u$ is appended. This process is repeated for all categorical features, resulting in a transformed $\mathbf{X}$ with $p_1$ binary features, where $p_1 \geq p$. Additionally, a label vector $\mathbf{Y} \in \mathbb{R}^{p_1}$ is created to track the transformation, where $\mathbf{Y}_i = u$ indicates that the $i$-th binary feature originated from the $u$-th categorical feature.
\item \textbf{Step 2 (Compute Temporal Graphs Per Timestamp)}: For each timestamp $t = 1, \ldots, T$, calculate the total number of samples at $t$ as:
\begin{align*}
n_t= \sum_{i=1,\ldots,n} 1(\tau_i=t).
\end{align*}
Next, compute the graph $\mathbf{A}_t$ for each timestamp $t$, defined as:
\begin{align*}
\mathbf{A}_{t}(u,v)= \frac{\sum_{i=1,\ldots,n}^{\tau_i=t}1(\mathbf{X}_{iu}=\mathbf{X}_{iv}=1)}{n_t} \in [0,1],
\end{align*}
where $u, v = 1, \ldots, p_1$. In this graph, $\mathbf{A}_t(u, v)$ represents the normalized co-occurrence frequency of binary features $u$ and $v$ at timestamp $t$.
\item \textbf{Step 3 (Compute Graph Covariance)}: For each timestamp $t$, calculate the $p_1 \times p_1$ covariance matrix as:
\begin{align*}
\Sigma_{t}(u,v)= \mathbf{A}_{t}(u,v) - \mathbf{A}_{t}(u,u)\mathbf{A}_{t}(v,v) \in [-1,1]
\end{align*}
for all $(u, v)$ such that $\mathbf{Y}_u \neq \mathbf{Y}_v$. Each $\Sigma_t(u, v)$ quantifies the dependence between features $u$ and $v$ at timestamp $t$.
\item \textbf{Step 4 (Determine $95\%$ Significance Level)}: For each timestamp $t$, compute a significance level $\epsilon_t$ using:
\begin{align*}
0.025 =2 Prob(\mbox{Normal}(0, 0.25) \leq \epsilon_t \sqrt{n_t}).
\end{align*}
This can be efficiently computed using the inverse cumulative distribution function.
\item \textbf{Step 5.1 (Detect Common Interactions)}: For each pair $(u, v)$, calculate:
\begin{align*}
\delta_1(u,v)&= \sum_{t=1}^{T} |\Sigma_{t}(u,v)| 1(|\Sigma_{t}(u,v)| > \epsilon_t) / T
\end{align*}
and identify the timestamps where the interaction is significant, along with the trend information (positive or negative), into the set:
\begin{align*}
t_s(u,v)&=\{t\cdot 1(\Sigma_{t}(u,v) > \epsilon_t) - t  \cdot 1(\Sigma_{t}(u,v) < -\epsilon_t)\}.
\end{align*}
Select the top $m$ pairs with the largest $\delta_1(u, v)$ and store them in the set $\Delta_1 = \{(u, v, t_s(u, v))\}$.
\item \textbf{Step 5.2 (Detect Spiked Interactions)}: For each pair $(u, v)$, calculate:
\begin{align*}
\delta_2(u,v)&= \sum_{t=1}^{T} |\Sigma_{t}(u,v)| 1(|\Sigma_{t}(u,v)| > \epsilon_t) / |t_s(u,v)|.
\end{align*}
Select the top $m$ pairs with the largest $\delta_2(u, v)$ and store them in the set $\Delta_2 = \{(u, v, t_s(u, v))\}$.
\item \textbf{Step 6 (LLM Query on Set Information)}: For each element in $\Delta_1$ and $\Delta_2$, query an LLM with the following details: basic dataset information, the feature pair $(u, v)$ along with their respective textual descriptions, the time range $t_s(u, v)$ during which the pair $(u, v)$ is significant, and the positive or negative trend information. Then, prompt the LLM to identify potential underlying events or contexts that could explain the significant relationship between $u$ and $v$ within the specified time range.
\end{itemize}

\subsection*{Detailed Discussions}
Step 1 employs the one-hot encoding scheme, transforming all categorical features into binary features. After this step, $\mathbf{X}$ contains only binary features, with $n$ samples and $p_1$ binary columns. This transformation preserves the categorical information by expanding each category into its own column, a technique fundamental in neural networks and machine learning \cite{Bishop1995, Murphy2012}. More recently, it has also been applied in hypothesis testing \cite{DCorKSampleTest} and graph embedding \cite{GEEOne}.

In Step 2, $\mathbf{A}_{t}$ measures how often two features are both equal to $1$, forming a graph of size $p_1 \times p_1$. Numerically, this computation is highly efficient due to the binary nature of the features. Specifically, for timestamp $t$, extracting $\mathcal{X}_{(t)} = \{\mathbf{X}_{iu} | \tau_i = t\} \in \mathbb{R}^{n_t \times p_1}$ allows the matrix $\mathbf{A}_{t}$ to be computed using matrix multiplication as:
\begin{align*}
\mathbf{A}_{t}= \mathcal{X}_{(t)}^{T} \mathcal{X}_{(t)} \in [0,1]^{p_1 \times p_1}.
\end{align*}

Step 3 centers the weighted graph computed in Step 2, resulting in the graph covariance matrix $\Sigma_t \in [-1, 1]^{p_1 \times p_1}$. While $\Sigma_t$ can still be interpreted as a weighted graph, it also serves as a quantitative measure of dependence between binary features. Importantly, when both features $u$ and $v$ originate from the same categorical variable, $\Sigma_t(u, v)$ is set to $0$ by default. This adjustment is necessary because one-hot encoding makes features from the same categorical variable mutually exclusive. As a result, computing the graph covariance for such pairs would always yield $\mathbf{A}_t(u, v) = 0$ and $\Sigma_t(u, v)$ as negative. Including these pairs would cause the algorithm to detect spurious negative trends arising solely from the binarization process.

In Step 4, we use a $95\%$ significance level based on the asymptotic null distribution derived in Theorem~\ref{thm1}, which assumes the largest possible variance. This level can be adjusted for different levels of sensitivity. For instance, a $99\%$ significance level would increase $\epsilon_t$, while a $90\%$ significance level would reduce it, allowing for more or less extreme level on the detected interactions.

Step 5 introduces two thresholding schemes to identify significant feature pairs: one for detecting the most common interactions ($\delta_1$) and another for capturing spikes in dependence ($\delta_2$). The distinction lies in the denominator: $\delta_1(u, v)$ averages over all timestamps, while $\delta_2(u, v)$ averages only over the timestamps where the interaction is significant. These two sets are not mutually exclusive; feature pairs exhibiting strong relationships over multiple years can appear in both sets. Additionally, the directional trend information, indicating whether the interaction is positive or negative, is stored in the sets to facilitate the LLM query in Step 6.

Alternative schemes for graph covariance can be designed to meet specific analytical needs, such as detecting exclusively positive or negative trends, identifying transition points where trends change, or highlighting pairs with consistently minimal relationships. For the purpose of this paper and the CTDC dataset, the focus is on the most common interactions and sudden spikes in feature dependence. However, the method's flexibility allows for modifications tailored to specific datasets or user objectives, enabling enriched analysis and deeper insights through LLM-powered interpretation.


\subsection*{Computational Complexity and Method Applicability}

Excluding the LLM query step, the entire algorithm has a computational complexity of $O(np_1)$, which scales linearly with the total number of samples. In contrast, substituting the graph covariance with a universal consistent dependence measure, while potentially producing comparable results, would be computationally prohibitive. Such measures typically require $O(n^2)$ complexity for each feature pair, rendering them impractical for large datasets. Even the fastest algorithm, such as the univariate distance correlation \cite{Hu2018, DCorFastTest}, requires $O(n \log n)$ complexity and falls short in speed compared to the proposed method. Detailed comparisons of actual running times are provided in the simulation and real data sections.

Step 6, which involves querying LLMs, is independent of the dataset size. Generating insights for each extracted feature pair typically requires only a few seconds when using popular LLM engines.

Although this method is designed for temporal categorical data, it can easily be applied to datasets without timestamps by treating all records as occurring at the same time. For datasets with continuous timestamps, users can define time ranges, such as yearly, monthly, or weekly intervals. Additionally, continuous features can be accommodated by discretizing their range into categories, effectively transforming them into categorical features. While these extensions are beyond the scope of this paper, which focuses on temporal categorical data like the CTDC dataset, the method is inherently versatile and applicable to a broader range of data formats.

\section{Mathematical Theory}

In this section, we quantify the proposed graph covariance using Bernoulli random variables and prove that the graph covariance is a consistent statistic for measuring dependence between binary features.

Without loss of generality, we focus on the graph element $\mathbf{A}_{t}(u,v)$ and the graph covariance $\Sigma_{t}(u,v)$ between two binary variables. Suppose $(\mathbf{X}_{iu},\mathbf{X}_{iv}) \stackrel{i.i.d.}{\sim} F_{X_u^{t} X_v^{t}}$ for each $i$ where $\tau_i=t$, with each marginal distribution assumed to be Bernoulli. Specifically, we have:
\begin{align*}
&X_{u_{t}} \sim Bernoulli(Prob(X_{u_{t}}=1)),\\ 
&X_{v_{t}} \sim Bernoulli(Prob(X_{v_{t}}=1)).
\end{align*}
It follows that:
\begin{align*}
1(X_{u_{t}}=X_{v_{t}}=1) \sim Bernoulli(Prob(X_{u_{t}}=X_{v_{t}}=1)).
\end{align*}
The probability can take any value between 0 and 1, depending on the dependence structure between $X_{u_{t}}$ and $X_{v_{t}}$. In the special case where $X_{u_{t}}$ and $X_{v_{t}}$ are independent at time $t$, we have:
\begin{align*}
Prob(X_{u_{t}}=X_{v_{t}}=1) &= Prob(X_{u_{t}}=1) Prob(X_{v_{t}}=1).
\end{align*}
Under this setting, the main theorem (the proof can be found in the appendix) is stated as follows:

\begin{restatable}{theorem}{thmOne}
\label{thm1}
At any given timestamp $t$:
\begin{itemize}
\item The graph covariance satisfies the law of large numbers such that:
\begin{align*}
\Sigma_{t}(u,v)&\stackrel{n_t \rightarrow \infty}{\rightarrow} \mu_{t}
\end{align*}
where 
\begin{align*}
\mu_t&= Prob(X_{u_{t}}=X_{v_{t}}=1) \\
&- Prob(X_{u_{t}}=1)Prob(X_{v_{t}}=1)).
\end{align*}
\item The graph covariance satisfies the central limit theorem such that:
\begin{align*}
\sqrt{n_t}(\Sigma_{t}(u,v)-\mu_{t}) \stackrel{dist}{\rightarrow}  \mbox{Normal}(0, \sigma_t^2)
\end{align*}
where the variance $\sigma_t^2$ satisfies: 
\begin{align*}
\sigma_t^2 \leq \frac{1}{4} + \frac{1}{16\sqrt{n_t}}.
\end{align*}
\item The graph covariance converges to $0$, i.e., $\mu_t = 0$, if and only if $X_{u_t}$ and $X_{v_t}$ are independent. 
\end{itemize}
\end{restatable}
Therefore, when the graph covariance deviates significantly from zero, it indicates the presence of dependence, while values close to zero suggest independence between features. The graph covariance quantifies the difference between the joint probability and the product of the marginal probabilities, allowing its magnitude to detect and measure changes in the strength of dependence over time.

Since the mean is $0$ under independence and the variance does not exceed $0.25$ for large $n_t$, this justifies the significance level $\epsilon_t$ defined in the main method. Alternatively, an approximate p-value can be derived based on $\text{Normal}(0, 0.25)$ as:
\begin{align*}
\mbox{p-val}=2 Prob(\sqrt{n_t}\Sigma_{t}(u,v) > \mbox{Normal}(0, 0.25)).
\end{align*}
Any graph covariance with a p-value exceeding the type 1 error level, e.g., $0.05$, can be considered insignificant, indicating that the two features are likely independent.

It is worth noting that the variance bound is conservative, as the maximum variance occurs when the Bernoulli probability equals $0.5$. In practice, the actual variance is typically smaller, making the level $\epsilon_t$ generally conservative as well.

\section{Simulation Study}
In this section, we present a series of simulations to demonstrate the ability of the proposed graph covariance method to detect dependence between features and capture dynamic interactions for both stable and shifting relationships over time.

We considered $t = 1, 2, \ldots, 8$ time steps, with $n_t = 5000$ samples generated for each $t$. Eight binary features were simulated, each exhibiting varying degrees of dependence on a reference feature, $X_{1_t}$. The generation process for these features is defined as follows:
\begin{itemize}
\item $X_{1_t} \sim Bernoulli(0.2)$ for all $t$, serving as the baseline feature for comparison.
\item $X_{2_t} \sim Bernoulli(0.1)$ for all $t$, remaining independent of $X_{1_t}$ across all time points.
\item $X_{3_t} \sim Bernoulli(0.9) * X_{1_t}$ for all $t$, exhibiting strong dependence on $X_{1_t}$ throughout time.
\item $X_{4_t} \sim Bernoulli(0.1 + 0.1 t) * X_{1_t}$ for all $t$, with increasing dependence on $X_{1_t}$ as time progresses.
\item $X_{5_t} \sim Bernoulli(0.9 - 0.1 t) * X_{1_t}$ for all $t$, showing decreasing dependence on $X_{1_t}$ over time.
\item $X_{6_t} \sim Bernoulli(0.1 + 0.2 (t-4)) * X_{1_t}$ for all $t$, demonstrating shifting dependence, decreasing from $t=1$ to $4$ and increasing from $t=4$ to $8$.
\item $X_{7_t} \sim \text{Bernoulli}(0.1)$ for $t < 8$, and $X_{7_t} \sim \text{Bernoulli}(0.9) \cdot X_{1_t}$ for $t = 8$, remaining independent of $X_{1_t}$ until a dependence spike at $t=8$.
\end{itemize}
These configurations simulate various scenarios where pairwise feature dynamics change over time, providing a comprehensive test of the graph covariance method’s capability to capture different types of feature interactions.

Figure~\ref{fig1} illustrates the behavior of the graph covariance statistic $\Sigma_t(1, u)$ for each pair of features relative to the baseline $X_{1_t}$, alongside the $95\%$ significance level. The results confirm the effectiveness of the proposed method in capturing intended dependence patterns:
\begin{itemize}
\item Independence: $\Sigma_t(1, 2)$ remains near zero across all time points, validating the independence between $X_{1_t}$ and $X_{2_t}$. The values consistently fall within the significance level, ensuring this pair is not extracted as either a common or spiked feature pair.
\item Stable Dependence: $\Sigma_t(1, 3)$ is positive and stable over time, reflecting the strong dependence between $X_{1_t}$ and $X_{3_t}$. The values remain consistently above the significance level.
\item Increasing Dependence: $\Sigma_t(1, 4)$ steadily increases with time, accurately capturing the growing dependence between $X_{1_t}$ and $X_{4_t}$.
\item Decreasing Dependence: $\Sigma_t(1, 5)$ gradually decreases over time, reflecting the weakening dependence between $X_{1_t}$ and $X_{5_t}$.
\item Shifting Dependence: $\Sigma_t(1, 6)$ first decreases from $t=1$ to $4$, then increases from $t=4$ to $8$, matching the intended shifting dependence between $X_{1_t}$ and $X_{6_t}$.
\item Dependence Spike: $\Sigma_t(1, 7)$ stays near zero and insignificant until $t=8$, where it shows a sharp spike above the significance level, capturing the sudden increase in dependence between $X_{1_t}$ and $X_{7_t}$.
\end{itemize}
These results demonstrate the graph covariance’s effectiveness in detecting various dependence patterns over time. If we apply Step 5.1 and Step 5.2, the feature pair with the most common interaction is $(1, 3)$, while the pair with the strongest spike interaction is $(1, 7)$.

In addition to accurately detecting temporal dependence shifts, the graph covariance method is highly computationally efficient. For this simulation, computing all graph covariances takes just $0.01$ seconds. By comparison, other dependence measures, such as distance correlation and the Hilbert-Schmidt independence criterion, produce similar results (see Figure~\ref{fig4} in the appendix) but require more time. The $O(n \log n)$ distance correlation takes $6$ seconds to compute all pairwise correlations across the $8$ time steps, while the $O(n^2)$ Hilbert-Schmidt independence criterion requires over two minutes.


Overall, the simulation results highlight the effectiveness and computational efficiency of the graph covariance method in detecting and quantifying changes in dependence between binary features over time. The proposed method successfully captures a variety of dependence structures - including increasing, decreasing, shifting, and spike relationships - with minimal computational burden.

\begin{figure}[htbp]
	\centering
	\includegraphics[width=0.95\textwidth,trim={1cm 0cm 1cm 0cm},clip]{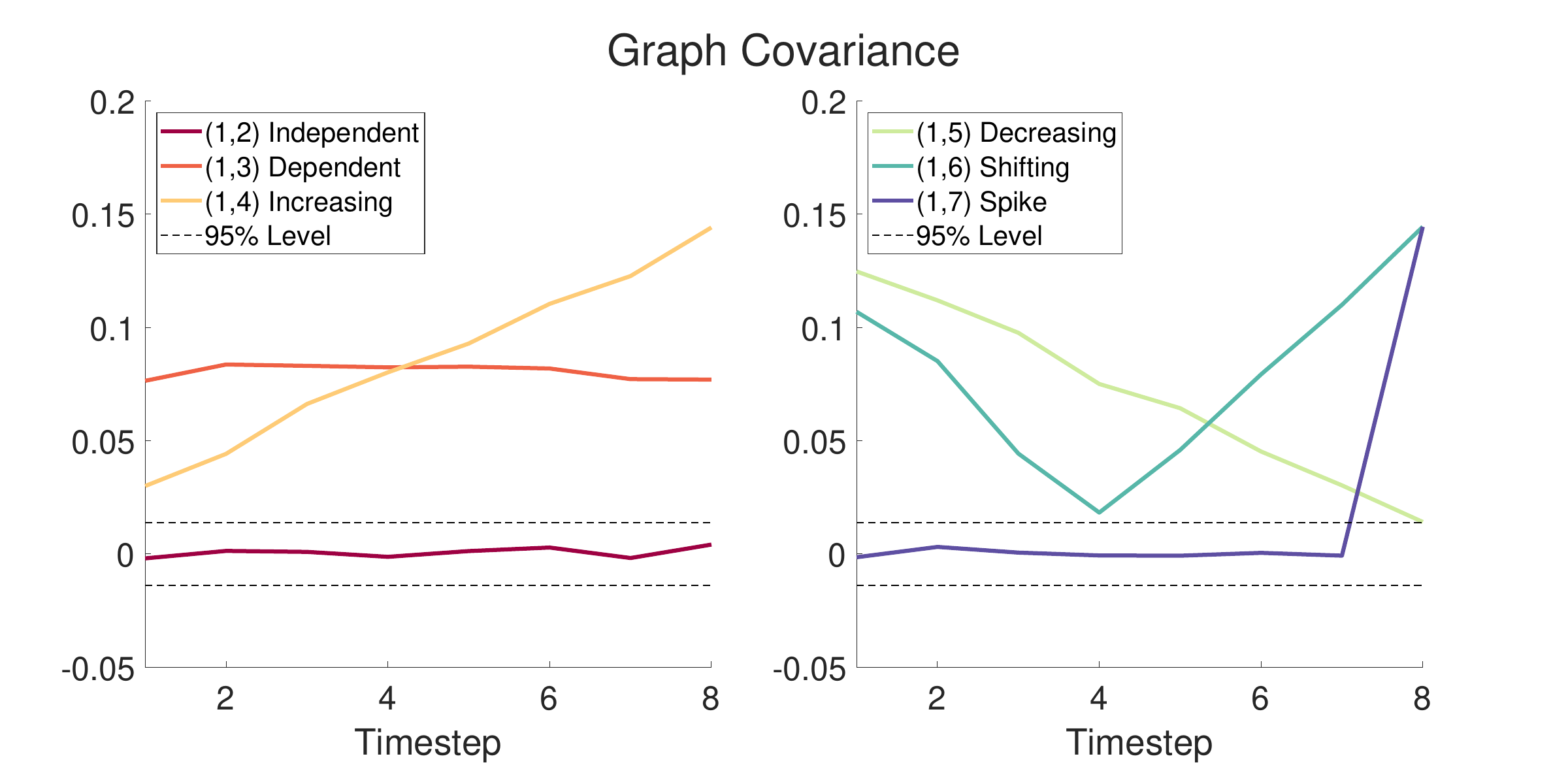}
	\caption{This figure shows the graph covariance between the baseline feature and the other simulated binary features across all timestamps.}
	\label{fig1}
\end{figure}

\section{Exploring the CTDC Dataset}

We applied the proposed method to the Counter-Trafficking Data Collaborative (CTDC) global synthetic dataset. After excluding approximately $30,000$ records that lacked a recorded year, the dataset comprised $n = 206,582$ records, each containing $p = 26$ categorical features along with the year of registration, and missing features (NaN) for samples were simply set to 0.

All categorical features were transformed into binary format using one-hot encoding, resulting in $p_1 = 158$ binary features. The graph covariance matrix, $\Sigma_t \in [-1, 1]^{158 \times 158}$, was computed for each year from 2002 to 2022, spanning $T = 21$ time steps. Using $m = 10$ feature pairs for both $\Delta_1$ (most common interactions) and $\Delta_2$ (spiked interactions), we identified the top significant pairs. These feature pairs were then passed to an LLM to generate detailed data insights, including information about the feature pairs, their textual descriptions, and their time-specific significance trends (positive or negative). The entire computation, excluding the LLM query, was completed in under 2 seconds.

\subsection{Computing Graph Covariance}
Figure~\ref{fig2} presents the graph covariance matrices $\Sigma_t$ for the years 2003, 2009, 2015, and 2021, illustrating the dependence structures that enable the extraction of significant feature pairs. While certain patterns remain stable, noticeable variations in the dependence structures are observed across different years, particularly between earlier and later time frames. These visualizations emphasize the evolving relationships among features over time. Note that entries between binary pairs within the same categorical variable are set to $0$, resulting in a substantial number of zero entries in the matrices. 

While the proposed method completes its computations in under two seconds, alternative dependence measures are considerably slower. For instance, fast distance correlation requires hours for this dataset due to the need to compute pairwise correlations for thousands of binary pairs, with each pair taking over $5$ seconds. Other dependence measures would take days to complete. Even with parallel computation, these methods remain significantly slower than the proposed graph covariance approach.

\begin{figure}[htbp]
	\centering
	\includegraphics[width=0.95\textwidth,trim={3cm 0cm 1cm 0cm},clip]{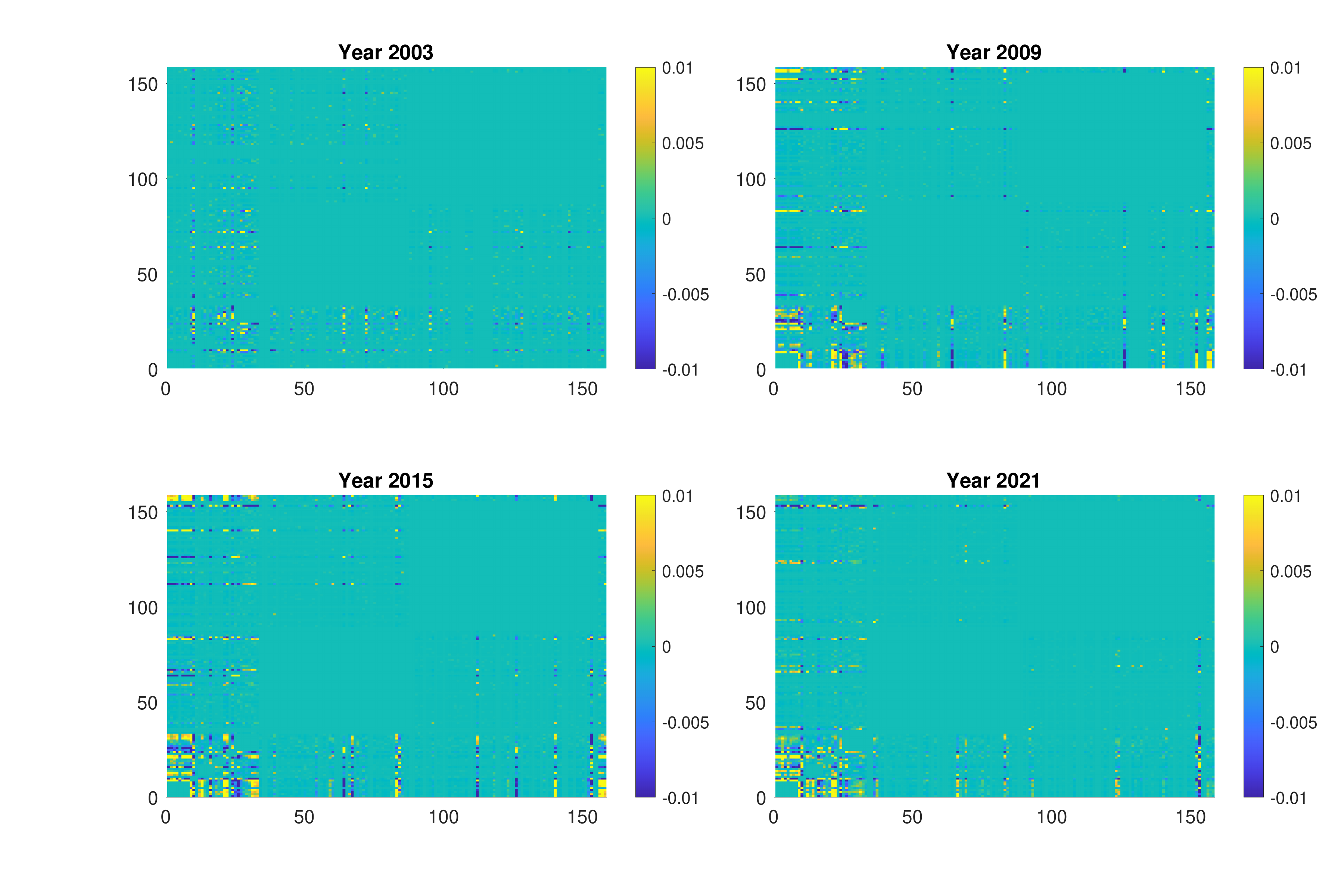}
	\caption{This figure visualizes the graph covariance matrix for the years 2003, 2009, 2015, and 2021 from the CTDC data.}
	\label{fig2}
\end{figure}

\subsection{Extracting Significant Feature Pairs}
Table~\ref{table1} presents the top 10 feature pairs in $\Delta_1$, representing the most significant and persistent trends, while Table~\ref{table2} showcases the top 10 pairs in $\Delta_2$, capturing spiked interactions. Although there is some overlap between the sets, $\Delta_1$ predominantly includes pairs significant over multiple years, whereas $\Delta_2$ includes several noteworthy spiked dependencies. Most relationships exhibit positive trends, with a few negative trends indicated by a negative sign preceding the years.

Most feature pairs in $\Delta_1$ align with widely recognized patterns in human trafficking. For example, the top pair (isSexualExploit, Woman) is significant across all years, reflecting the prevalence of sexual exploitation among women. Similarly, (meansDenyBasicNeeds, isForcedLabour) highlights the frequent association of forced labor with the denial of basic needs. In contrast, (isSexualExploit, Man) is significantly negative, fitting the gendered nature of sexual exploitation, while (isForcedLabour, Man) shows a positive and significant relationship.

Figure~\ref{fig3} visualizes graph covariances for selected feature pairs from $\Delta_1$ and $\Delta_2$, as well as additional pairs, compared against the $95\%$ significance level (varying due to sample size differences by year). The top-left panel shows the positive trend for (isForcedLabour, Man), consistently above the threshold in most years. The top-center panel highlights (isSexualExploit, Woman), with strong significance except for a small dip around 2008–2014. The top-right panel confirms the strong relationship between (isForcedLabour, recruiterRelationOther), unaffected by gender.

The bottom-left panel highlights notable country pairs, including $(MDA, MDA)$ from $\Delta_1$, consistently significant until 2017, and $(GHA, GHA)$ and $(MMR, IDN)$ from $\Delta_2$, which exhibit sudden spikes in 2005 and 2015, respectively. The bottom-center and bottom-right panels examine additional pairs, such as $(Woman, UKR)$, which is slightly significant around 2008, and $(Woman, USA)$, slightly significant in 2015. Minimal relationships are observed between features such as means of false promise and citizenship, providing an example of independent feature pairs.


\begin{table}[h]
\renewcommand{\arraystretch}{1.3}
\centering
\small
{\begin{tabular}{|c|c|}
 \hline
 Pairs & Years\\
\hline
(isSexualExploit, Woman) & 2002-2022 \\
(MDA,MDA) & 2004-2017\\
(isForcedLabour, recruiterRelationOther) & 2006-2022\\
(isSexualExploit, USA) & 2015-2022\\
(UKR, UKR) & 2007-2022\\
(isSexualExploit, typeOfSexProstitution) & 2015-2022\\
(meansDenyBasicNeeds,isForcedLabour) & 2008-2022\\
(meansDebtBondageEarnings,isForcedLabour) & 2009-2022\\
(isSexualExploit, Man) & -2005-2022\\
(isForcedLabour, Man) & 2005-2022\\
\hline
\end{tabular}
\caption{This table presents the top 10 feature pairs in the set $\Delta_1$ along with their timestamps of significance (for clarity, single insignificant years within a range are omitted). For country-related pairs, "MDA" represents Moldova, and "UKR" represents Ukraine. In these pairs, the first country corresponds to the categorical attribute Citizenship, while the second country corresponds to the categorical attribute Country of Exploitation. }
\label{table1}
}
\end{table}

\begin{table}[h]
\renewcommand{\arraystretch}{1.3}
\centering
\small
{\begin{tabular}{|c|c|}
 \hline
 Pairs & Year\\
\hline
(PHL,USA) & -2016\\
(isOtherExploit, PHL) & 2016\\
(isSexualExploit, USA) & 2015-2022\\
(isSexualExploit, PHL) & -2016\\
(isSexualExploit, typeOfSexProstitution) & 2015-2022\\
(isSexualExploit, Woman) & 2002-2022 \\
(GHA,GHA) & 2005\\
(MDA,MDA) & 2004-2017\\
(MMR,IDN) & 2015\\
(isForcedLabour, GHA) & 2005\\
\hline
\end{tabular}
\caption{This table lists the top 10 feature pairs in the set $\Delta_1$. For country-related pairs, "PHL" represents the Philippines, "GHA" represents Ghana, while "MMR" and "IDN" represent Myanmar and Indonesia, respectively.}
\label{table2}
}
\end{table}

\begin{figure*}[htbp]
	\centering
	\includegraphics[width=0.85\textwidth,trim={1cm 0cm 1cm 0cm},clip]{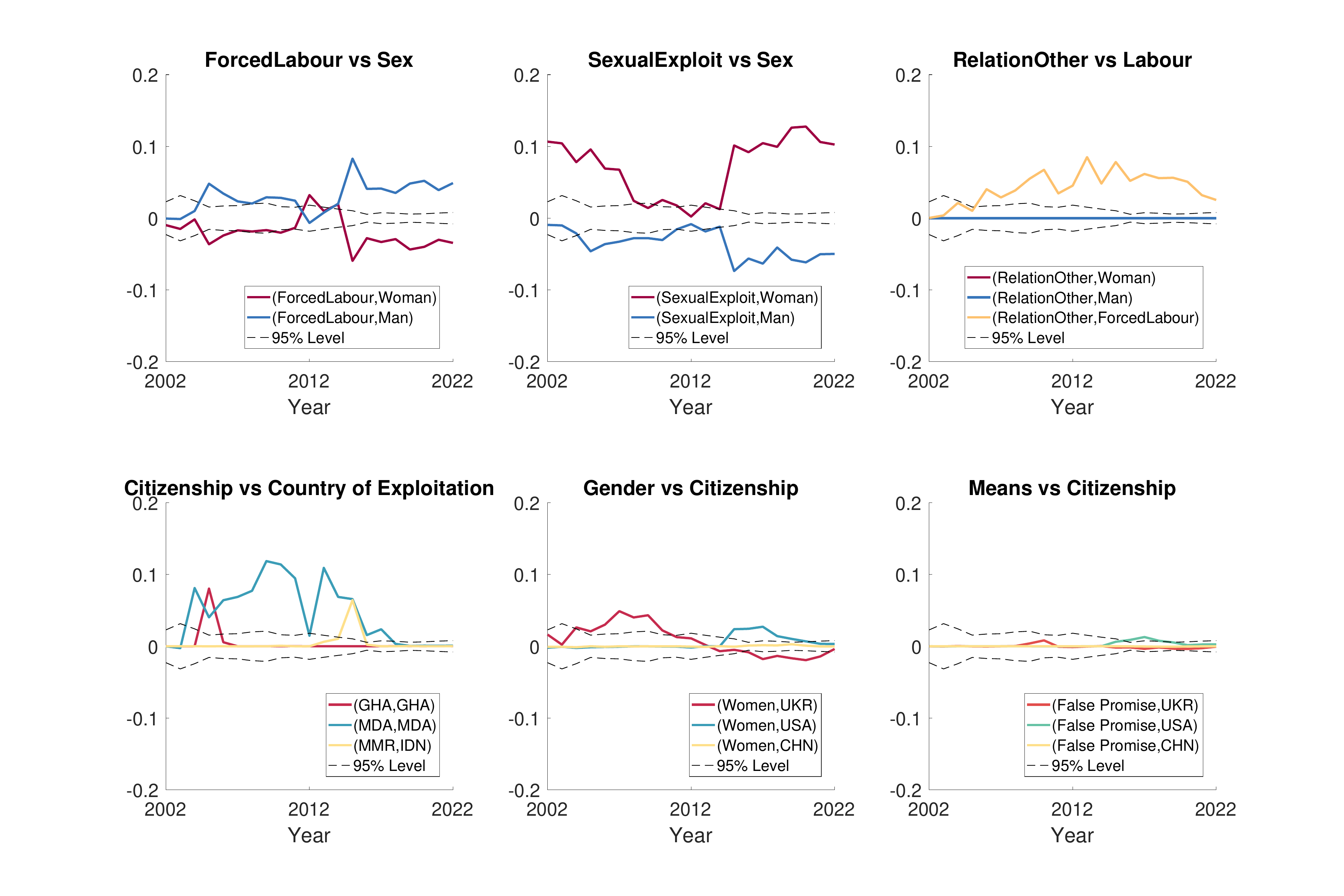}
	\caption{This figure visualizes the graph covariance and its changes over time for several noteworthy pairs of features.}
	\label{fig3}
\end{figure*}


\subsection{Explain Interactions using LLM Insights}

While the graph covariance and feature extraction yield intriguing feature pairs and corresponding timestamps, and the visualizations in Figure~\ref{fig3} are valuable for further analysis, understanding and explaining the observed interactions across years often requires significant effort to investigate underlying events, whether global or local. Some pairs, especially those in $\Delta_1$, are relatively straightforward to interpret. However, pairs in $\Delta_2$, due to their spiked nature, are more challenging to explain without historical context. These spikes could stem from data collection anomalies, such as a surge in records during a particular year, or from local or geopolitical events.

Leveraging LLMs, our tool generates immediate data insights for any feature pair, whether from $\Delta_1$, $\Delta_2$, or user-specified queries. For example, we queried the following into an LLM: "In the Counter-Trafficking Data Collaborative dataset, \textbf{after year 2015}, there are \textbf{more} reported cases based on \textbf{SexualExploit} and \textbf{USA Country of Exploitation}. Please analyze the potential causes, global and local events, and geopolitical factors that may explain this trend." This query is straightforward and can be customized by replacing the bolded text with any other feature pair. The LLM typically produces a detailed analysis spanning 1–2 pages, often concluding with a summary. Here, we report only the final conclusions generated with a GPT-4 LLM instance:

\textit{"The rise in reported cases of Sexual Exploitation and USA as the Country of Exploitation after 2015 can be attributed to a combination of factors: Increased awareness and reporting of trafficking cases post-2015; Strengthened laws and victim protections in the U.S.; the role of the U.S. as a major destination for trafficking, particularly for sexual exploitation; increased global migration and refugee crises after 2015, which put vulnerable populations at greater risk; Advances in international cooperation and data-sharing efforts. "}

As another example, we queried the negative interaction between (PHL, USA) in 2016 using the following prompt: "In the Counter-Trafficking Data Collaborative dataset, \textbf{at year 2016}, there are \textbf{less} reported cases based on \textbf{Philippines Citizen} and \textbf{USA Country of Exploitation}. Please analyze the potential causes, global and local events, and geopolitical factors that may explain this trend." A brief summary of key events provided by ChatGPT-4 includes:

\textit{"The reduction in reported trafficking cases involving Philippine citizens and the U.S. as a country of exploitation in 2016 can likely be attributed to a combination of factors: Stronger migration controls, tighter visa regulations, and increased law enforcement efforts in both the Philippines and the U.S. likely curtailed trafficking flows to the U.S. Diversified migration destinations and regional geopolitical shifts, such as the increased demand for labor in the Middle East and Southeast Asia, may have reduced the flow of Filipino workers to the U.S. Increased awareness, improved victim identification, and better anti-trafficking measures in both the Philippines and the U.S. helped to address trafficking cases before they reached the U.S. Changes in data reporting or a possible underreporting of cases may have also played a role in this trend."}

Additional output summary on a few interesting feature pairs are included in the appendix. 

\section{Conclusion}
In this paper, we proposed a fast and scalable method for statistical analysis and data insight generation using large language models (LLMs) for large-scale categorical datasets with timestamps, focusing on the CTDC global synthetic dataset as a case study. The method's primary strengths lie in its scalability, computational efficiency, and ability to construct correlation matrices, detect trends, and identify outliers in temporal data. By integrating graph covariance with LLM queries, the approach enables actionable insights by uncovering potential events and trends from dependence shifts, facilitating real-time analysis and interpretation of complex datasets.

Theoretical analysis established that graph covariance is a consistent measure of dependence between binary random variables, effectively quantifying the strength of dependence through probability differences. Simulation studies validated its accuracy in capturing dynamic dependence patterns, and the application to the CTDC dataset demonstrated its capability to deliver rapid, interpretable analysis. By combining graph covariance results with LLM-generated insights, the method sheds light on socio-political and economic factors driving observed trends, supporting enhanced decision-making and providing a richer understanding of complex interactions.

Looking ahead, several avenues for future work present exciting opportunities. Theoretical extensions could explore the relationship between graph covariance and other dependence measures, as well as its adaptation to continuous features by discretizing ranges into quantiles. Methodological advancements might include generalizing the approach to handle triplet interactions, offering richer insights and improving the accuracy of LLM interpretations. For example, in the CTDC case study, significant triplets such as (PHL, USA, isSexualExploit) in 2016 and (isForcedLabour, GHA, GHA) in 2005 can be manually identified. Incorporating additional features or contextual information into LLM queries can further improve the precision and relevance of LLM outputs. Such extensions present promising opportunities to enhance the method's versatility and effectiveness in uncovering data insights from complex datasets.


\section*{Acknowledgment}
This work was supported in part the National Science Foundation DMS-2113099, and by funding from Microsoft Research.

\bibliographystyle{abbrvnat} 
\bibliography{shen,general}

\begin{thebibliography}{31}
\providecommand{\natexlab}[1]{#1}
\providecommand{\url}[1]{\texttt{#1}}
\expandafter\ifx\csname urlstyle\endcsname\relax
  \providecommand{\doi}[1]{doi: #1}\else
  \providecommand{\doi}{doi: \begingroup \urlstyle{rm}\Url}\fi

\bibitem[Al-Tammemi et~al.(2023)Al-Tammemi, Nadeem, Kutkut, Ali, Angawi, Abdallah, Abutaima, Shoumar, Albakri, and Sallam]{AlTammemi2023}
A.~B. Al-Tammemi, A.~Nadeem, L.~Kutkut, M.~Ali, K.~Angawi, M.~H. Abdallah, R.~Abutaima, R.~Shoumar, R.~Albakri, and M.~Sallam.
\newblock Are we seeing the unseen of human trafficking? a retrospective analysis of the ctdc k-anonymized global victim of trafficking data pool in the period 2010–2020.
\newblock \emph{PLOS ONE}, 18\penalty0 (4):\penalty0 e0284762, 2023.

\bibitem[Bishop(1995)]{Bishop1995}
C.~M. Bishop.
\newblock \emph{Neural Networks for Pattern Recognition}.
\newblock Oxford University Press, 1995.

\bibitem[Box et~al.(2015)Box, Jenkins, Reinsel, and Ljung]{box2015time}
G.~E.~P. Box, G.~M. Jenkins, G.~C. Reinsel, and G.~M. Ljung.
\newblock \emph{Time Series Analysis: Forecasting and Control}.
\newblock John Wiley \& Sons, 2015.

\bibitem[Chatterjee(2021)]{Chatterjee2021}
S.~Chatterjee.
\newblock A new coefficient of correlation.
\newblock \emph{Journal of the American Statistical Association}, 116\penalty0 (536):\penalty0 2009--2022, 2021.

\bibitem[Chaudhuri and Hu(2019)]{Hu2018}
A.~Chaudhuri and W.~Hu.
\newblock A fast algorithm for computing distance correlation.
\newblock \emph{Computational Statistics and Data Analysis}, 135:\penalty0 15--24, 2019.

\bibitem[Devlin et~al.(2019)Devlin, Chang, Lee, and Toutanova]{Devlin2019BERTPO}
J.~Devlin, M.-W. Chang, K.~Lee, and K.~Toutanova.
\newblock Bert: Pre-training of deep bidirectional transformers for language understanding.
\newblock In \emph{North American Chapter of the Association for Computational Linguistics}, 2019.

\bibitem[Edelmann and Goeman(2022)]{Edelmann2022}
D.~Edelmann and J.~Goeman.
\newblock A regression perspective on generalized distance covariance and the hilbert–schmidt independence criterion.
\newblock \emph{Statistical Science}, 37\penalty0 (4):\penalty0 562--579, 2022.

\bibitem[Edge et~al.(2020)Edge, Yang, Lytvynets, Cook, Galez-Davis, Darnton, and White]{Edge2020}
D.~Edge, W.~Yang, K.~Lytvynets, H.~Cook, C.~Galez-Davis, H.~Darnton, and C.~M. White.
\newblock Design of a privacy-preserving data platform for collaboration against human trafficking.
\newblock \emph{arXiv preprint arXiv:2005.05688}, 2020.
\newblock URL \url{https://arxiv.org/abs/2005.05688}.

\bibitem[Enders(2010)]{enders2010applied}
W.~Enders.
\newblock \emph{Applied Econometric Time Series}.
\newblock John Wiley \& Sons, 2010.

\bibitem[Eryarsoy et~al.(2024)Eryarsoy, Topuz, and Demiroglu]{eryarsoy2024disentangling}
E.~Eryarsoy, K.~Topuz, and C.~Demiroglu.
\newblock Disentangling human trafficking types and the identification of pathways to forced labor and sex: an explainable analytics approach.
\newblock \emph{Annals of Operations Research}, 335\penalty0 (2):\penalty0 761--795, 2024.

\bibitem[Gretton et~al.(2005)Gretton, Herbrich, Smola, Bousquet, and Scholkopf]{GrettonEtAl2005}
A.~Gretton, R.~Herbrich, A.~Smola, O.~Bousquet, and B.~Scholkopf.
\newblock Kernel methods for measuring independence.
\newblock \emph{Journal of Machine Learning Research}, 6:\penalty0 2075--2129, 2005.

\bibitem[Heller et~al.(2013)Heller, Heller, and Gorfine]{HellerGorfine2013}
R.~Heller, Y.~Heller, and M.~Gorfine.
\newblock A consistent multivariate test of association based on ranks of distances.
\newblock \emph{Biometrika}, 100\penalty0 (2):\penalty0 503--510, 2013.

\bibitem[Kunihama and Dunson(2013)]{Kunihama2013}
T.~Kunihama and D.~B. Dunson.
\newblock Bayesian modeling of temporal dependence in large sparse contingency tables.
\newblock \emph{Journal of the American Statistical Association}, 108\penalty0 (504):\penalty0 741--750, 2013.

\bibitem[Li et~al.(2022)Li, Bruce, and Cai]{Li2022}
Z.~Li, S.~A. Bruce, and T.~Cai.
\newblock Interpretable classification of categorical time series using the spectral envelope and optimal scalings.
\newblock \emph{Journal of Machine Learning Research}, 23:\penalty0 1--31, 2022.

\bibitem[Murphy(2012)]{Murphy2012}
K.~P. Murphy.
\newblock \emph{Machine Learning: A Probabilistic Perspective}.
\newblock MIT Press, 2012.

\bibitem[Olisah et~al.(2024)Olisah, Odooh, Efijemue, Obunadike, Onwuchekwa, Owolabi, Akintayo, and Obunadike]{Olisah2024}
S.~Olisah, C.~Odooh, O.~Efijemue, E.~Obunadike, J.~Onwuchekwa, O.~Owolabi, S.~Akintayo, and C.~Obunadike.
\newblock Unveiling global human trafficking trends: A comprehensive analysis.
\newblock \emph{Journal of Data Analysis and Information Processing}, 12:\penalty0 49--75, 2024.

\bibitem[Panda et~al.(2025)Panda, Shen, Perry, Zorn, Lutz, Priebe, and Vogelstein]{DCorKSampleTest}
S.~Panda, C.~Shen, R.~Perry, J.~Zorn, A.~Lutz, C.~E. Priebe, and J.~T. Vogelstein.
\newblock Universally consistent k-sample tests via dependence measures.
\newblock \emph{Statistics and Probability Letters}, 216:\penalty0 110278, 2025.

\bibitem[Radford et~al.(2019)Radford, Wu, Child, Luan, Amodei, and Sutskever]{radford2019languagegpt2}
A.~Radford, J.~Wu, R.~Child, D.~Luan, D.~Amodei, and I.~Sutskever.
\newblock Language models are unsupervised multitask learners.
\newblock \emph{OpenAI Blog}, 1\penalty0 (8):\penalty0 1--18, 2019.

\bibitem[Raffel et~al.(2020)Raffel, Shazeer, Roberts, Lee, Narang, Matena, Zhou, Li, and Liu]{Raffel2019ExploringTL}
C.~Raffel, N.~Shazeer, A.~Roberts, K.~Lee, S.~Narang, M.~Matena, Y.~Zhou, W.~Li, and P.~J. Liu.
\newblock Exploring the limits of transfer learning with a unified text-to-text transformer.
\newblock \emph{Journal of Machine Learning Research}, 21\penalty0 (140):\penalty0 1--67, 2020.

\bibitem[Shen et~al.(2020)Shen, Priebe, and Vogelstein]{MGCDCor}
C.~Shen, C.~E. Priebe, and J.~T. Vogelstein.
\newblock From distance correlation to multiscale graph correlation.
\newblock \emph{Journal of the American Statistical Association}, 115\penalty0 (529):\penalty0 280--291, 2020.

\bibitem[Shen et~al.(2022)Shen, Panda, and Vogelstein]{DCorFastTest}
C.~Shen, S.~Panda, and J.~T. Vogelstein.
\newblock The chi-square test of distance correlation.
\newblock \emph{Journal of Computational and Graphical Statistics}, 31\penalty0 (1):\penalty0 254--262, 2022.

\bibitem[Shen et~al.(2023)Shen, Wang, and Priebe]{GEEOne}
C.~Shen, Q.~Wang, and C.~E. Priebe.
\newblock One-hot graph encoder embedding.
\newblock \emph{IEEE Transactions on Pattern Analysis and Machine Intelligence}, 45\penalty0 (6):\penalty0 7933--7938, 2023.

\bibitem[Shen et~al.(2024)Shen, Chung, Mehta, Xu, and Vogelstein]{DCorTemporalTest}
C.~Shen, J.~Chung, R.~Mehta, T.~Xu, and J.~T. Vogelstein.
\newblock Independence testing for temporal data.
\newblock \emph{Transactions on Machine Learning Research}, 2024.
\newblock ISSN 2835-8856.
\newblock URL \url{https://openreview.net/forum?id=jv1aPQINc4}.

\bibitem[Shumway and Stoffer(2010)]{shumway2010time}
R.~H. Shumway and D.~S. Stoffer.
\newblock \emph{Time Series Analysis and Its Applications: With R Examples}.
\newblock Springer Science \& Business Media, 2010.

\bibitem[Stoffer et~al.(1993)Stoffer, Tyler, and McDougall]{Stoffer1993}
D.~S. Stoffer, D.~E. Tyler, and A.~J. McDougall.
\newblock Spectral analysis for categorical time series: Scaling and the spectral envelope.
\newblock \emph{Biometrika}, 80\penalty0 (3):\penalty0 611--622, 1993.

\bibitem[Szekely et~al.(2007)Szekely, Rizzo, and Bakirov]{SzekelyRizzoBakirov2007}
G.~Szekely, M.~Rizzo, and N.~Bakirov.
\newblock Measuring and testing independence by correlation of distances.
\newblock \emph{Annals of Statistics}, 35\penalty0 (6):\penalty0 2769--2794, 2007.

\bibitem[Vaswani et~al.(2017)Vaswani, Shazeer, Parmar, Uszkoreit, Jones, Gomez, Kaiser, and Polosukhin]{vaswani2017attention}
A.~Vaswani, N.~Shazeer, N.~Parmar, J.~Uszkoreit, L.~Jones, A.~N. Gomez, {\L}.~Kaiser, and I.~Polosukhin.
\newblock Attention is all you need.
\newblock In \emph{Advances in Neural Information Processing Systems}, volume~30. Curran Associates, Inc., 2017.

\bibitem[Walby and Francis(2024)]{walby2024improving}
S.~Walby and B.~Francis.
\newblock Improving the estimate of trafficking in human beings and modern slavery by integrating data from ilo/walk free/iom and unodc.
\newblock \emph{Social Indicators Research}, pages 1--25, 2024.

\bibitem[Wang et~al.(2021)Wang, Li, and Zhu]{hsicts}
G.~Wang, W.~Li, and K.~Zhu.
\newblock New {HSIC}-based tests for independence between two stationary multivariate time series.
\newblock \emph{Statistica Sinica}, 31\penalty0 (1):\penalty0 269--300, 2021.

\bibitem[Zhang(2024)]{zhang2024}
Q.~Zhang.
\newblock On the properties of distance covariance for categorical data: Robustness, sure screening, and approximate null distributions.
\newblock \emph{arXiv preprint arXiv:2403.17882}, 2024.

\bibitem[Zhou(2012)]{Zhou2012}
Z.~Zhou.
\newblock Measuring nonlinear dependence in time‐series, a distance correlation approach.
\newblock \emph{Journal of Time Series Analysis}, 33\penalty0 (3):\penalty0 438--457, 2012.

\end{thebibliography}


\clearpage
\appendix
\setcounter{figure}{0}
\renewcommand{\thealgorithm}{C\arabic{algorithm}}
\renewcommand{\thefigure}{E\arabic{figure}}
\renewcommand{\thesubsection}{\thesection.\arabic{subsection}}
\renewcommand{\thesubsubsection}{\thesubsection.\arabic{subsubsection}}
\pagenumbering{arabic}
\renewcommand{\thepage}{\arabic{page}}

\bigskip
\begin{center}
{\large\bf APPENDIX}
\end{center}
\section{Proofs}
\label{sec:proofs}

\thmOne*
\begin{proof}
(i): Based on basic probability and the law of large numbers, the normalized graph element $\mathbf{A}_t(u,v)$ satisfies:
\begin{align*}
\mathbf{A}_t (u,v) &\stackrel{n_t \rightarrow \infty}{\rightarrow} Prob(X_{u_{t}}=X_{v_{t}}=1). 
\end{align*}
Moreover, we have
\begin{align*}
\mathbf{A}_t (u,u) &\stackrel{n_t \rightarrow \infty}{\rightarrow} Prob(X_{u_{t}}=1), \\
\mathbf{A}_t (v,v) &\stackrel{n_t \rightarrow \infty}{\rightarrow} Prob(X_{v_{t}}=1).
\end{align*}
It follows that:
\begin{align*}
\Sigma_{t}(u,v) &= (\mathbf{A}_t (u,v)- \mathbf{A}_t (u,u) \mathbf{A}_t (v,v)) \\ 
&\rightarrow  Prob(X_{u_{t}}=X_{v_{t}}=1) - Prob(X_{u_{t}}=1)Prob(X_{v_{t}}=1) \\
&=  \mu_t.
\end{align*}

(ii): Applying the central limit theorem, it is clear that 
\begin{align*}
\sqrt{n_t} (\mathbf{A}_{t}(u,v)-Prob(X_{u_{t}}=X_{v_{t}}=1)) \stackrel{dist}{\rightarrow} \mbox{Normal}(0, Prob(X_{u_{t}}=X_{v_{t}}=1) (1-Prob(X_{u_{t}}=X_{v_{t}}=1))).
\end{align*}
Similarly, we have 
\begin{align*}
&\sqrt{n_t} (\mathbf{A}_{t}(u,u)-Prob(X_{u_{t}}=1)) \stackrel{dist}{\rightarrow} \mbox{Normal}(0, Prob(X_{u_{t}}=1)(1-Prob(X_{u_{t}}=1))),\\
&\sqrt{n_t} (\mathbf{A}_{t}(v,v)-Prob(X_{v_{t}}=1)) \stackrel{dist}{\rightarrow} \mbox{Normal}(0, Prob(X_{v_{t}}=1)(1-Prob(X_{v_{t}}=1))).
\end{align*}
It follows that:
\begin{align*}
\sqrt{n_t}(\Sigma_{t}(u,v)-\mu_{t}) = \sqrt{n_t}(\mathbf{A}_t (u,v)- \mathbf{A}_t (u,u) \mathbf{A}_t (v,v) -\mu_{t})\stackrel{dist}{\rightarrow}  \mbox{Normal}(0, \sigma_t^2).
\end{align*}
where 
\begin{align*}
\sigma_t^2 &\leq Prob(X_{u_{t}}=X_{v_{t}}=1) (1-Prob(X_{u_{t}}=X_{v_{t}}=1)) \\
&+ Prob(X_{u_{t}}=1)(1-Prob(X_{u_{t}}=1))Prob(X_{v_{t}}=1)(1-Prob(X_{v_{t}}=1)) / \sqrt{n_t}\\
&\leq \frac{1}{4} + \frac{1}{16\sqrt{n_t}}
\end{align*}

(iii): It suffices to prove that:
\begin{align*}
Prob(X_{u_{t}}=X_{v_{t}}=1) = Prob(X_{u_{t}}=1)Prob(X_{v_{t}}=1)
\end{align*}
if and only if $X_{u_{t}}$ and $X_{v_{t}}$ are independent. 

The if direction is trivial. To prove the only if direction, since $X_{u_{t}}$ and $X_{v_{t}}$ are binary, their independence is equivalent to proving the following four equalities: 
\begin{align*}
Prob(X_{u_{t}}=X_{v_{t}}=1) = Prob(X_{u_{t}}=1)Prob(X_{v_{t}}=1)\\
Prob(X_{u_{t}}=0, X_{v_{t}}=1) = Prob(X_{u_{t}}=0)Prob(X_{v_{t}}=1)\\
Prob(X_{u_{t}}=1, X_{v_{t}}=0) = Prob(X_{u_{t}}=1)Prob(X_{v_{t}}=0)\\
Prob(X_{u_{t}}=X_{v_{t}}=0) = Prob(X_{u_{t}}=0)Prob(X_{v_{t}}=0).
\end{align*}
The first equality is given and holds by assumption. The second equality holds because:
\begin{align*}
&\mbox{Prob}(X_{u_{t}}=0, X_{v_{t}}=1) \\
&= \mbox{Prob}(X_{v_{t}}=1) - \mbox{Prob}(X_{u_{t}}=X_{v_{t}}=1)\\
&=\mbox{Prob}(X_{v_{t}}=1) - \mbox{Prob}(X_{u_{t}}=1) \mbox{Prob}(X_{v_{t}}=1) \\
&=\mbox{Prob}(X_{v_{t}}=1) (1-\mbox{Prob}(X_{u_{t}}=1))\\
&=\mbox{Prob}(X_{u_{t}}=0) \mbox{Prob}(X_{v_{t}}=1).
\end{align*}
By the same argument, but switching $X_{u_{t}}$ and $X_{v_{t}}$, we obtain the third equality. The fourth equality follows from:
\begin{align*}
&\mbox{Prob}(X_{u_{t}}=X_{v_{t}}=0) \\
&= \mbox{Prob}(X_{v_{t}}=0) - \mbox{Prob}(X_{u_{t}}=0, X_{v_{t}}=0)\\
&=\mbox{Prob}(X_{v_{t}}=0) - \mbox{Prob}(X_{u_{t}}=1) \mbox{Prob}(X_{v_{t}}=0) \\
&=\mbox{Prob}(X_{v_{t}}=0) (1-\mbox{Prob}(X_{u_{t}}=1))\\
&=\mbox{Prob}(X_{u_{t}}=0) \mbox{Prob}(X_{v_{t}}=0).
\end{align*} 
Thus, we have proven the only if direction. 
\end{proof}

\section{Additional Simulations}
\label{sec:addSims}

Figure~\ref{fig4} replicates the experiment from Figure~\ref{fig1}, using distance correlation and the Hilbert-Schmidt Independence Criterion (HSIC). While these methods produce quantitatively similar results to the graph covariance shown in Figure~\ref{fig1}, they are significantly slower in computation time: $0.01$ seconds for graph covariance, $6$ seconds for distance correlation, and $2$ minutes for HSIC.

\begin{figure*}[htbp]
	\centering
	\includegraphics[width=0.8\textwidth,trim={1cm 0cm 1cm 0cm},clip]{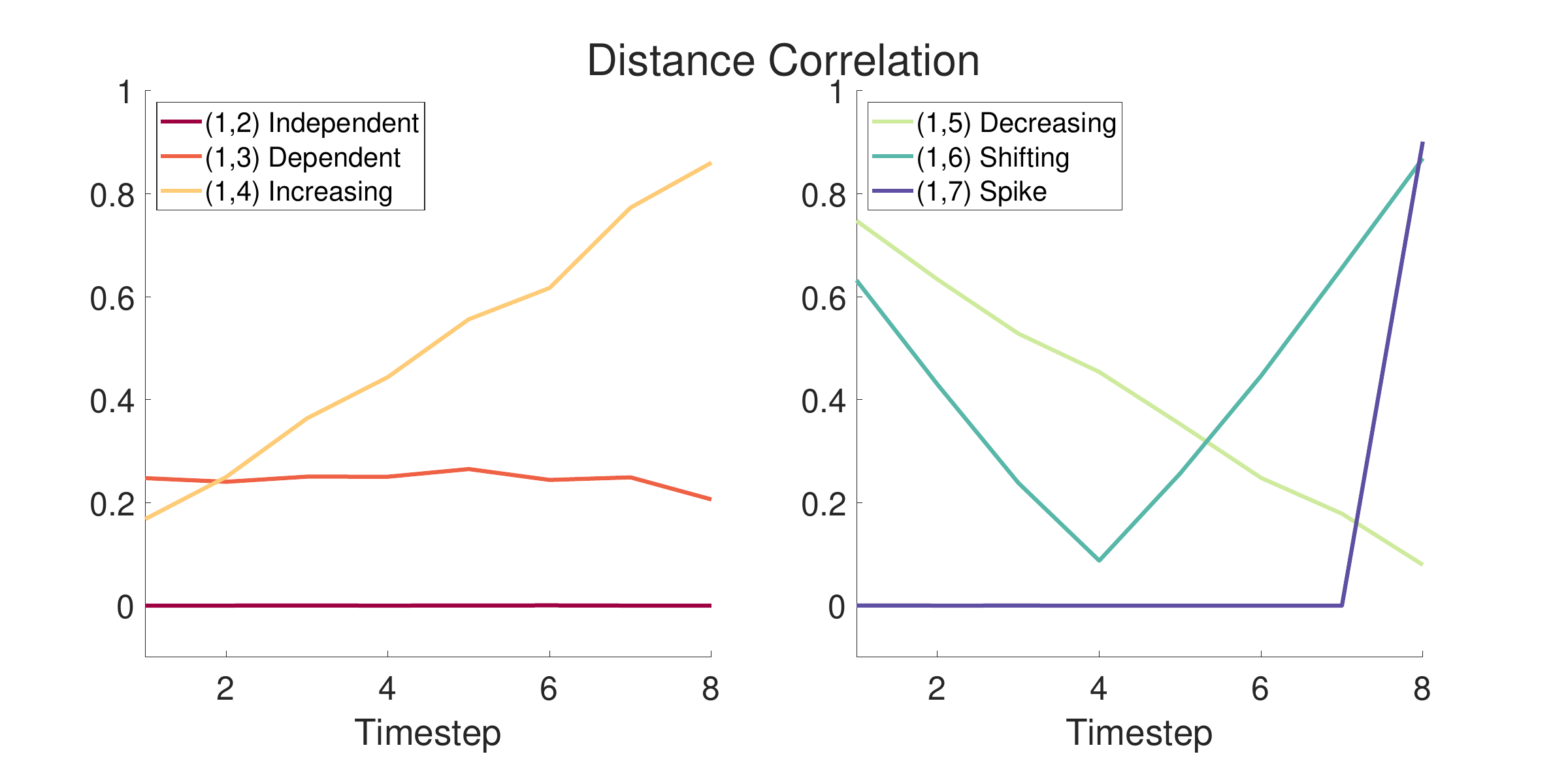}
    \includegraphics[width=0.8\textwidth,trim={1cm 0cm 1cm 0cm},clip]{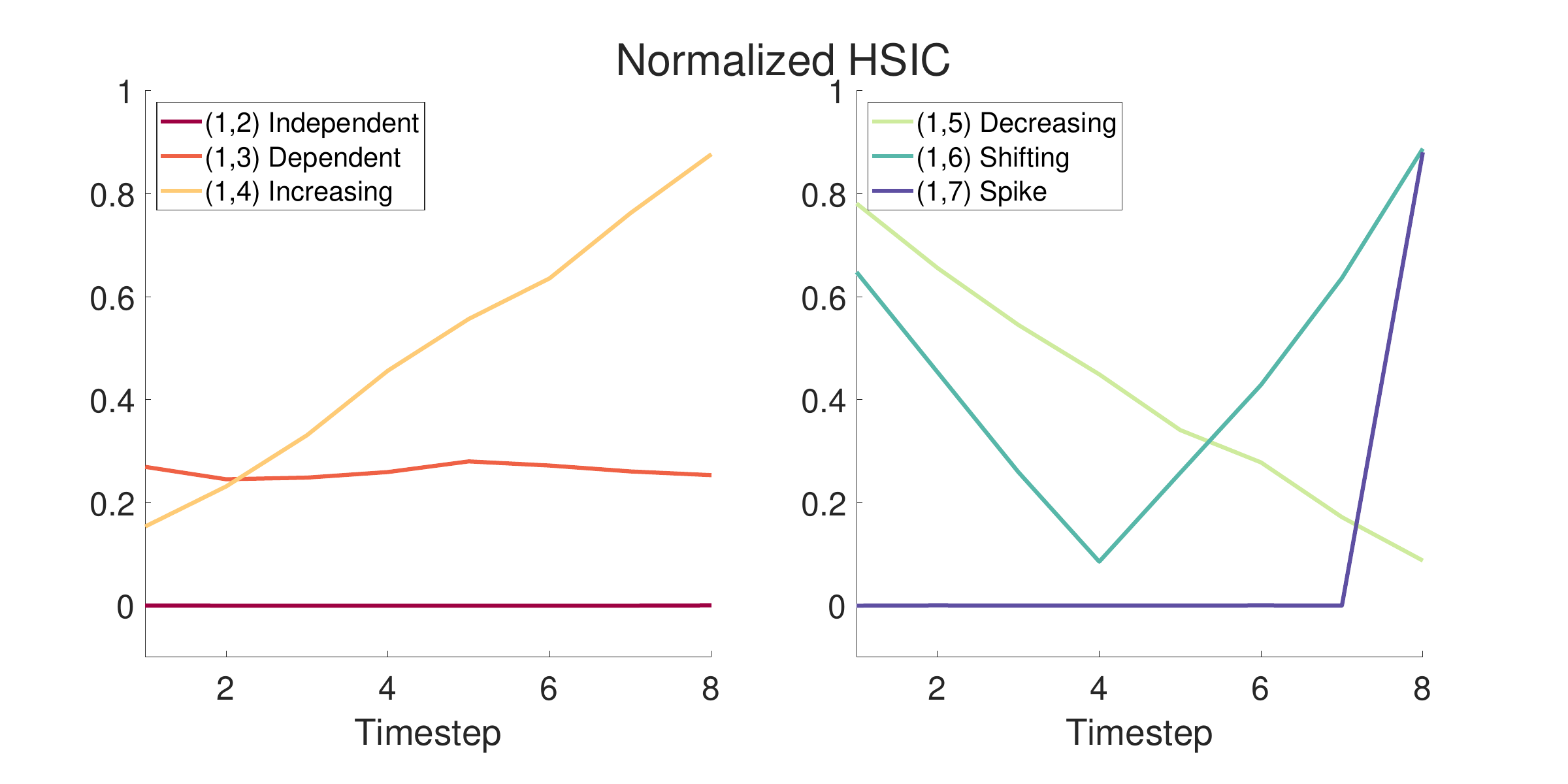}
	\caption{Same experiments as Figure~\ref{fig1} but using distance correlation (top two panels) and Hilbert-Schmidt independence criterion (bottom two panels).}
	\label{fig4}
\end{figure*}

\section{LLM Query Details on CTDC Data}

Here, we present a few LLM-generated data insights for additional feature pairs not included in the main paper. Although any feature pair can be queried, and the LLM outputs are far more detailed than the summaries provided here, these examples are included for illustrative purposes.

\subsection{LLM for Feature Pair (MMR, IDN)}
Query: "In the Counter-Trafficking Data Collaborative dataset, \textbf{at year 2015}, there are \textbf{more} reported cases based on \textbf{Myanmar Citizen} and \textbf{Indonesia  Country of Exploitation}. Please analyze the potential causes, global and local events, and geopolitical factors that may explain this trend."\\

\noindent
Output Summary: \textit{"The increase in reported trafficking cases involving Myanmar citizens and Indonesia as the country of exploitation in 2015 is likely due to a combination of local, regional, and global factors, including: Myanmar’s political and economic instability, which led to increased migration and vulnerability to trafficking. Indonesia’s demand for migrant labor, particularly in the fishing industry, domestic work, and agriculture, where exploitation was rampant. The Rohingya crisis exacerbating migration flows from Myanmar, increasing vulnerability to trafficking. The global focus on human trafficking and increased media attention to exploitative industries in Indonesia, such as fishing. Improved victim identification and better reporting mechanisms in the region during this period. These factors, combined with regional migration patterns and geopolitical dynamics, likely contributed to the higher number of trafficking cases reported in the CTDC dataset for 2015."}

\subsection{LLM for Feature Pair (isForcedLabour, GHA)}

Query: "In the Counter-Trafficking Data Collaborative dataset, \textbf{at year 2005}, there are \textbf{more} reported cases based on \textbf{forced Labour} and \textbf{Ghana Country of Exploitation}. Please analyze the potential causes, global and local events, and geopolitical factors that may explain this trend."\\

\noindent
Output Summary: \textit{"The rise in reported cases of forced labour in Ghana in 2005 can be attributed to a combination of local, regional, and global factors: Local vulnerabilities, particularly in the cocoa industry, where child labour and trafficking were significant issues. Regional migration and trafficking patterns, with people being moved across West Africa and exploited for labour in Ghana. International efforts to combat forced labour and human trafficking, resulting in increased awareness and improved data reporting mechanisms. Geopolitical shifts, such as the attention to Sub-Saharan Africa in international human rights discourses, and Ghana's involvement in global anti-trafficking initiatives. Together, these factors likely contributed to the noticeable increase in reported forced labour cases in Ghana in 2005."}

\subsection{LLM for Feature Pair (women, UKR)}

While the feature pair (Woman, UKR) is not included in $\Delta_1$ or $\Delta_2$, it represents an intriguing case from the bottom center panel of Figure~\ref{fig3}, showing a positive relationship before 2010 and a negative relationship after 2015. A two-part LLM query can be constructed to investigate the underlying causes, as follows:\\

\noindent
Query: "In the Counter-Trafficking Data Collaborative dataset, \textbf{before 2010}, there are \textbf{more} reported cases based on \textbf{women} and \textbf{Ukrain Country of Exploitation}. However, \textbf{after 2015}, there are \textbf{less} reported cases based on \textbf{women} and \textbf{Ukrain Country of Exploitation}. Please analyze the potential causes, global and local events, and geopolitical factors that may explain this trend."\\

\noindent
Output Summary: \textit{"Before 2010, higher reported cases of trafficking involving women in Ukraine as a country of exploitation were driven by economic instability, high unemployment, and widespread poverty following the collapse of the Soviet Union. Weak anti-trafficking laws, coupled with active post-Soviet organized crime networks, made Ukraine a hub for trafficking, particularly for sexual exploitation. Women seeking better opportunities were targeted under false promises of work abroad, while patriarchal norms and systemic gender inequality exacerbated their vulnerabilities. Increased international attention and reporting by NGOs further highlighted the scale of the issue during this period."}\\

\noindent
Output Summary: \textit{"After 2015, reported cases declined due to several factors, including geopolitical changes such as the Euromaidan revolution, the annexation of Crimea, and the conflict in eastern Ukraine, which altered trafficking routes and displaced populations. Ukraine’s closer alignment with the EU led to stronger border controls, improved anti-trafficking legislation, and disrupted trafficking networks. However, the ongoing conflict and shifting reporting mechanisms may have masked the full extent of trafficking, as victims from Ukraine were increasingly exploited abroad rather than domestically. Economic stabilization efforts and international aid also contributed to reducing vulnerabilities during this period."}

\end{document}